\documentclass{colt2013}

\usepackage{graphicx}
\usepackage{algorithm}
\usepackage{algorithmic}
\usepackage{hyperref}
\usepackage{amssymb, multirow, paralist, amsmath, bm}

\newtheorem{thm}{Theorem}

\DeclareMathOperator*{\argmin}{argmin}

\def \x {\mathbf{x}}
\def \z {\mathbf{z}}

\def \y {\mathbf{y}}
\def \w {\mathbf{w}}
\def \R {\mathbb{R}}

\def \g {\mathbf{g}}
\def \gh {\hat{\g}}

\def \gb {\bar{\g}}
\def \f {\mathbf{f}}

\def \fb {\bar{\f}}

\def \E {\mathrm{E}}

\def \D {\mathcal {D}}
\def \H {\mathcal {H}}

\def \Z {\mathcal {Z}}
\def \F {\mathcal {F}}

\def \U {\mathcal {U}}
\makeatletter
\AtBeginDocument{\Hy@breaklinkstrue}
\makeatother
\begin{document}

%\title[Short Title]{Full Title of Article}
\title[$O(\log T)$ Projections for Stochastic Optimization]{$O(\log T)$ Projections for Stochastic Optimization of \\
Smooth and Strongly Convex Functions}

 \coltauthor{\Name{Lijun Zhang} \Email{zhanglij@msu.edu} \\
   \addr Department of Computer Science and Engineering\\
Michigan State University, East Lansing, MI 48824, USA\\
\Name{Tianbao Yang} \Email{tyang@ge.com}\\
\addr Machine Learning Lab, GE Global Research,\\
San Ramon, CA 94583, USA\\
    \Name{Rong Jin} \Email{rongjin@cse.msu.edu}\\
   \addr Department of Computer Science and Engineering\\
Michigan State University, East Lansing, MI 48824, USA\\
\Name{Xiaofei He} \Email{xiaofeihe@cad.zju.edu.cn}\\
\addr State Key Laboratory of CAD\&CG, College of Computer Science
  Zhejiang University, Hangzhou 310027, China}

\maketitle

\begin{abstract}
Traditional algorithms for stochastic optimization require projecting the solution at each iteration into a given domain to ensure its feasibility. When facing complex domains, such as positive semi-definite cones, the projection operation can be expensive, leading to a high computational cost per iteration. In this paper, we present a novel algorithm that aims to reduce the number of projections for stochastic optimization. The proposed algorithm combines the strength of several recent developments in stochastic optimization, including mini-batch, extra-gradient, and epoch gradient descent, in order to effectively explore the smoothness and strong convexity. We show, both in expectation and with a high probability, that when the objective function is both \emph{smooth} and \emph{strongly convex}, the proposed algorithm achieves the optimal $O(1/T)$ rate of convergence with only $O(\log T)$ projections. Our empirical study verifies the theoretical result.
\end{abstract}


\section{Introduction}
The goal of stochastic optimization is to solve the optimization problem
\[
\min_{\w \in \D} F(\w),
\]
using only the stochastic gradients of $F(\w)$. In particular, we assume there exists a gradient oracle, which for any point $\w \in \D$, returns a random vector $\gh(\w)$ that gives an unbiased estimate of the subgradient of $F(\cdot)$ at $\w$. A special case of stochastic optimization is the risk minimization problem, whose objective function is given by
\[
F(\w)=\E_{(\x,y)} \left[ \ell(\w;(\x,y)) \right],
\]
where $(\x,y)$ is an instance-label pair, $\ell$ is a convex loss function that measures the prediction error, and the expectation is taken oven the unknown joint distribution of $(\x,y)$~\citep{Zhang_SVM,COLT:Shalev:2009,NIPS2009_AGM}. The performance of stochastic optimization algorithms is typically characterized by the \emph{excess risk}
\[
F(\w_T) - \min_{\w \in \D} F(\w),
\]
where $T$ is the number of iterations and $\w_T$ is the solution obtained after making $T$ calls to the gradient oracle.

For general Lipschitz continuous convex functions, stochastic gradient descent exhibits the unimprovable $O(1/\sqrt{T})$ rate of convergence~\citep{Problem_Complexity,nemirovski-2008-robust}. For strongly-convex functions, the algorithms proposed in very recent works~\citep{Juditsky:SCS,COLT:Hazan:2011,ICML2012Rakhlin,NIPS2012_DualAve} achieve the optimal $O(1/T)$ rate~\citep{IT:SCO}. Although these convergence rates are significantly worse than the results in deterministic optimization, stochastic optimization is appealing due to its low per-iteration complexity. However, this is not the case when the domain $\D$ is complex. This is because most stochastic optimization algorithms require projecting the solution at each iteration into domain $\D$ to ensure its feasibility, an expensive operation when the domain is complex. In this paper, we show that if the objective function is smooth and strongly convex, it is possible to reduce the number of projections dramatically without affecting the convergence rate.

Our work is motivated by the difference in convergence rates between stochastic and deterministic optimization. When the objective function is smooth and convex, under the first-order oracle assumption, Nesterov's accelerated gradient method enjoys the optimal  $O(1/T^2)$ rate~\citep{nesterov2004introductory,Nesterov_Non_Smooth}. Thus, \emph{for deterministic optimization of smooth and convex functions, we can achieve an $O(1/\sqrt{T})$ rate by only performing $O(T^{1/4})$ updating}. When the objective function is smooth and strongly convex, the optimal rate for first-order algorithms is $O(1/\alpha^k)$, for some constant $\alpha >1$~\citep{nesterov2004introductory,Nesterov_Composite}. In other words, \emph{for deterministic optimization of smooth and strongly convex functions, we can achieve an $O(1/T)$ rate by only performing $O(\log T)$ updating}. The above observations inspire us to consider the following  questions.
\begin{compactenum}
  \item For Stochastic Optimization of Smooth and Convex functions (SOSC), is it possible to maintain the optimal $O(1/\sqrt{T})$ rate by performing $O(T^{1/4})$ projections?
  \item For Stochastic Optimization of Smooth and Strongly Convex functions (SOS$^2$C), is it possible to maintain the optimal $O(1/T)$ rate by performing $O(\log T)$ projections?
\end{compactenum}

For the 1st question, we have found a positive answer from literature. By combining mini-batches~\citep{NIPS2007_Topmoumoute} with the accelerated stochastic approximation~\citep{Lan:SCO}, we can achieve the optimal $O(1/\sqrt{T})$ rate by performing $O(T^{1/4})$ projections~\citep{NIPS2011_0942}. However, a naive application of mini-batches does not lead to the desired $O(\log T)$ complexity for SOS$^2$C. The main contribution of this paper is a novel stochastic optimization algorithm that answers the 2nd question positively.
% The proposed algorithm is built upon three techniques:
% \begin{compactenum}[i)]
%    \item mini-batches with exponentially increasing sizes, which makes it possible to reduce the number of projections from $O(T)$ to $O(\log T)$,
%   \item the extra-gradient descent~\citep{nemirovski-2005-prox,Nemirovski:SMP}, that allows us to explore the smoothness assumption and benefit from mini-batches,
%   \item the idea of dividing the learning process into a sequence of different epochs~\citep{COLT:Hazan:2011}, which can utilize the strongly convexity assumption.
% \end{compactenum}
Our theoretical analysis reveals, both in expectation and with a high probability, that the proposed algorithm achieves the optimal $O(1/T)$ rate by only performing $O(\log T)$ projections.

%Moreover, the rate upper bound is shown to hold in expectation, as well as in high probability up to a $poly(\log \log T)$ factor.
\section{Related Work}
In this section, we provide a brief review of the existing approaches for avoiding projections.
\subsection{Mini-batch based algorithms}
Instead of updating the solution after each call to the gradient oracle, mini-batch based algorithms use the average gradient over multiple calls to update the solution~\citep{NIPS2007_Topmoumoute,Shalev-ShwartzSSC11,ICML2011Dekel}.
For a fixed batch size $B$, the number of updates (and projections) is reduced from $O(T)$ to $O(T/B)$, and the variance of the stochastic gradient is reduced from $\sigma$ to $\sigma/\sqrt{B}$. By appropriately balancing between the loss cased by a smaller number of updates and the reduction in the variance of stochastic gradients, it is able to maintain the optimal rate of convergence.

The idea of mini-batches can be incorporated into any stochastic optimization algorithm that uses gradient-based updating rules. When the objective function is smooth and convex, combining mini-batches with the accelerated stochastic approximation~\citep{Lan:SCO} leads to
\[
O\left(\frac{B^2}{T^2} + \frac{1}{\sqrt{T}} \right)
\]
rate of convergence~\citep{NIPS2011_0942}. By setting $B=T^{3/4}$, we achieve the optimal $O(1/\sqrt{T})$ rate with only $O(T^{1/4})$ projections. When the target function is smooth and strongly convex, we can apply mini-batches to the optimal algorithms for strongly convex functions~\citep{NIPS2009_AGM,Lan:SCSC}, leading to
\[
O\left(\frac{B^2}{T^2} + \frac{1}{T} \right)
\]
rate of convergence~\citep{Dekel:2012:ODO}. In order to maintain the optimal $O(1/T)$ rate, the value of $B$ cannot be larger than $\sqrt{T}$, implying at least $O(\sqrt{T})$ projections are required. In contrast, the algorithm proposed in this paper achieves an $O(1/T)$ rate with only $O(\log T)$ projections.
\subsection{Projection free algorithms}
Due to the low iteration cost, Frank-Wolfe algorithm~\citep{Frank_Wolfe} or conditional gradient method~\citep{Conditional_Gradient} has seen a recent surge of interest in machine learning~\citep{Hazan:2008:SAS,Clarkson:2010:CSG,ICML2013Simon}. At each iteration of the Frank-Wolfe algorithm, instead of performing a projection that requires solving a constrained quadratic programming problem, it solves a constrained linear programming problem. For many domains of interest, including the positive semidefinite cone and the trace norm ball, the constrained linear problem can be solved more efficiently than a projection problem~\citep{ICML2013Jaggi}, making this kind of methods attractive for large-scale optimization.

In a recent work~\citep{ICML2012Hazan}, an online variant of the Frank-Wolfe algorithm is proposed. Although the online Frank-Wolfe algorithm exhibits an $O(1/\sqrt{T})$ convergence rate for smooth functions, it is unable to achieve the optimal $O(1/T)$ rate for strongly convex functions. Besides, the memory complexity of this algorithm is $O(T)$, making it unsuitable for large-scale optimization problems. Another related work is the stochastic gradient descent with only one projection~\citep{NIPS2012_OneP}. This algorithm is built upon the assumption that the solution domain can be characterized by an inequality constraint $g(\w) \leq 0$ and the gradient of $g(\cdot)$ can be evaluated efficiently. Unfortunately, this assumption does not hold for some commonly used domain (e.g., the trace norm ball). Compared to the projection free algorithms, our proposed method is more general because it make no assumption about the solution domain.

\section{Stochastic Optimization of Smooth and Strongly Convex Functions}
\subsection{Preliminaries}
We first define smoothness and strongly convexity.
\begin{definition}
A function $f: \D \rightarrow \R$ is $L$-smooth w.r.t. a norm $\|\cdot\|$ if $f$ is everywhere differentiable and
\[ % \label{eqn:smooth}
\| \nabla f(\w) - \nabla f(\w') \|_* \leq L \|\w - \w'\|, \ \forall \w, \w' \in \D.
\]
where $\|\cdot\|_*$ is the dual norm.
\end{definition}
\begin{definition}
A smooth function $f: \D \rightarrow \R$ is $\lambda$-strongly convex w.r.t. a norm $\|\cdot\|$, if $f$ is everywhere differentiable and
\[ % \label{eqn:convex}
\| \nabla f(\w) - \nabla f(\w') \|_* \geq \lambda \|\w - \w'\|, \ \forall \w, \w' \in \D.
\]
\end{definition}
To simplify our analysis, we assume that both $\|\cdot\|$ and $\|\cdot\|_*$ are the vector $\ell_2$ norm in the following discussion.

Following~\citep{COLT:Hazan:2011}, we make the following assumptions about the gradient oracle.
 \begin{compactitem}
  \item There is a gradient oracle, which, for a given input point $\w$ returns a stochastic gradient $\gh(\w)$ whose expectation is the gradient of  $F(\w)$ at $\w$, i.e.,
      \[
      \E[\gh(\w)] = \nabla F(\w).
      \]
      We further assume the stochastic gradients obtained by calling the oracle are \emph{independent}.
   \item The gradient oracle is $G$-bounded, i.e.,
      \[
      \|\gh(\w)\|\leq G, \ \forall \w \in \D.
      \]
      We note that this assumption may be relaxed by assuming the orlicz norm of $\gh(\w)$ to be bounded~\citep{Lan:SCO}, i.e., $\E[\exp(\|\gh(\w)\|^2/G^2)] \leq \exp(1)$. Although our theoretical result holds even under the assumption of bounded orlicz norm, we choose the $G$-bounded gradient for simplicity.
 \end{compactitem}

Define $\w_*$ as the optimal solution that minimizes $F(\w)$, i.e., $\w_* = \argmin_{\w \in \D} F(\w)$.
Using the strongly convexity of $F(\w)$, we have \citep{COLT:Hazan:2011}
\begin{equation} \label{eqn:1}
\frac{\lambda}{2} \|\w-\w_*\|^2 \leq F(\w)-F(\w_*) \leq \frac{2G^2}{\lambda},  \forall \ \w \in \D.\\
\end{equation}
\subsection{The Algorithm}
Algorithm~\ref{alg:3} shows the proposed method for Stochastic Optimization of Smooth and Strongly Convex functions (SOS$^2$C), that achieves the optimal $O(1/T)$ rate of convergence by performing $O(\log T)$ projections. The inputs of the algorithm are: (1) $\eta$, the step size, (2) $M$, the fixed number of updates per epoch/stage, (3) $B^1$, the initial batch size, and (4) $T$, the total number of calls to the gradient oracle. With a slight abuse of notation, we use $\gh(\w,i)$ to denote the stochastic gradient at $\w$ obtained after making the $i$-th call to the oracle. We denote the projection of $\w$ onto the domain $\D$ by $\Pi_{\D}(\w)$.

Similar to the epoch gradient descent algorithm~\citep{COLT:Hazan:2011}, the proposed algorithm consists of two layers of loops. It uses the outer ({\bfseries while}) loop to divide the learning process into a sequence of epochs (Step~5 to Step~12). Similar to~\citep{COLT:Hazan:2011}, the number of calls to the gradient oracle made by Algorithm~\ref{alg:3} increases exponentially over the epoches, a key that allows us to achieve the optimal $O(1/T)$ convergence rate for strongly convex functions. We note that other techniques, such as the $\alpha$-suffix averaging~\citep{ICML2012Rakhlin}, can also be used as an alternative.

In the inner ({\bfseries for}) loop of each epoch, we combine the idea of mini-batches~\citep{ICML2011Dekel} with extra-gradient descent~\citep{nemirovski-2005-prox,Nemirovski:SMP}. We choose extra-gradient descent because it allows us to replace in the excess risk bound $\E[\|\gh(\w)\|^2]$ with $\E[\|\gh(\w) - \E[\gh(\w)]\|^2]$, the variance of the stochastic gradient $\gh(\w)$, thus opening the door to fully exploring the capacity of mini-batches in variance reduction.

%In contrast, if we use the classic stochastic gradient descent, the upper bound will depend on the norm of the stochastic gradient~\citep{zinkevich-2003-online,COLT:Hazan:2011}, which cannot be reduced by mini-batches.

To be more specific, in the $k$-th epoch, we maintain two sequences of solutions $\{\w_t^k\}_{t=1}^M$ and $\{\z_{t}^k\}_{t=1}^M$, where $\z_t^k$ is an auxiliary solution that allows us to effectively explore the smoothness of the loss function. At each iteration $t$ of the $k$-th epoch, we calculate the average gradients $\gb_t^k$ and
$\fb_t^k$ by calling the gradient oracle $B^k$ times (Steps 6 and 8), and update the solutions $\w_t^k$ and $\z_t^k$ using the average gradients (Steps 7 and 9). The batch size $B^k$ is fixed inside each epoch but doubles from epoch to epoch (Step 11). This is in contrast to most mini-batch based algorithms that have a fixed batch size. This difference is critical for achieving $O(1/T)$ convergence rate with only $O(\log T)$ updates.

\begin{algorithm}[t]
\caption{$\log T$ Projections for SOS$^2$C}
\begin{algorithmic}[1]
\STATE {\bfseries Input:}  parameters $\eta$, $M$, $B^1$ and $T$
\STATE Initialize $\w_1^1 \in \D$ arbitrarily
\STATE Set $k=1$
\WHILE{$2 M \sum_{i=1}^k B^i \leq T$}
\FOR{$t=1$ to $M$} \label{stp:1}
\STATE Compute the average gradient at $\w_t^k$ over $B^k$ calls to the gradient oracle
\[
\gb_t^k = \frac{1}{B^k} \sum_{i=1}^{B^k} \gh(\w_t^k,i)
\]
\STATE Update
\[
\z_t^k=\Pi_{\D}\left(\w_t^k-\eta \gb_t^k  \right)
\]
\STATE Compute the average gradient at $\z_t^k$ over $B^k$ calls to the gradient oracle
\[
\fb_t^k = \frac{1}{B^k} \sum_{i=1}^{B^k}\gh(\z_t^k,i)
\]
\STATE Update
\[
\w_{t+1}^k=\Pi_{\D}\left(\w_t^k-\eta \fb_t^k  \right)
\]
\ENDFOR
\STATE $\w_1^{k+1}=\frac{1}{M} \sum_{t=1}^{M} \z_t^k$, and $B^{k+1}=2B^k$
\STATE $k=k+1$
\ENDWHILE
\STATE {\bfseries Return:} $\w_1^{k}$
\end{algorithmic}\label{alg:3}
\end{algorithm}

\subsection{The main results}
The following theorem bounds the expected excess risk of the solution return by Algorithm~\ref{alg:3} and the number of projections.

\begin{thm} \label{thm:1}
Set the parameters in Algorithm~\ref{alg:3} as
\begin{eqnarray*}
\eta = \frac{1}{\sqrt{6}L},  \  M  = \frac{4}{\eta \lambda} \textrm{ and } B^1 =12 \eta \lambda.
\end{eqnarray*}
The final point $\w_1^k$ returned by Algorithm~\ref{alg:3} makes at most $T$ calls to the gradient oracle, and has its excess risk bounded by
\[
\E[F(\w_1^k) - F(\w_*)] \leq \frac{384  G^2}{ \lambda T} = O \left(\frac{1}{T} \right),
\]
and the total number of projections bounded by
\[
\frac{8 \sqrt{6} L}{\lambda} \left \lfloor \log_2\left( \frac{T}{96}+1\right) \right \rfloor = O\left(\log T\right).
\]
\end{thm}
Theorem~{\ref{thm:1}} shows that in expectation, Algorithm~\ref{alg:3} achieve an $O(1/T)$ convergence with $O(\log T)$ updates. The following theorem gives a high probability bound of the excess risk for Algorithm~\ref{alg:3}.

\begin{thm} \label{thm:2}
Set the parameters in Algorithm~\ref{alg:3} as
\begin{eqnarray*}
\eta = \frac{1}{\sqrt{6}L},  \  M  = \frac{4}{\eta \lambda} \textrm{ and } B^1 =\alpha \eta \lambda,
\end{eqnarray*}
where $\alpha$  is defined below. For any $ 0 < \delta < 1$, let
\begin{align}
\tilde{\delta}=&\frac{\delta}{k^\dag}, \nonumber\\
k^\dag=&\left \lfloor \log_2\left( \frac{ T}{8 \alpha}+1\right) \right  \rfloor = O(\log T),\label{eqn:smooth:k} \\
\alpha  = & \max\left\{400 \log^2 \frac{8 M}{\tilde{\delta}} , 1 + 64 \log^2 \frac{8 M}{\tilde{\delta}}  \left( \log \frac{4N}{\tilde{\delta}} + \frac{4}{9} \log^2 \frac{4N}{\tilde{\delta}} \right) \right\} \label{eqn:smooth:alpha:2} \\
= & O \left[\left( \log \log T + \log \frac{1}{\delta}\right)^4\right], \nonumber \\
N =  & \left \lceil \log_2 \frac{4M T }{\eta \lambda} \right\rceil=O (\log T). \label{eqn:lemma:N}
\end{align}
The final point $\w_1^k$ returned by Algorithm~\ref{alg:3} makes at most $T$ calls to the gradient oracles, performs
\[
\frac{8 \sqrt{6} L}{\lambda} \left \lfloor \log_2\left( \frac{T}{8 \alpha}+1\right) \right \rfloor = O\left(  \log T\right)
\]
projections, and with a probability at least $1-\delta$, has its excess risk bounded by
\[
F(\w_1^k) - F(\w_*) \leq \frac{32 \alpha G^2}{ \lambda T} = O\left( \frac{ ( \log \log T + \log 1/\delta )^4}{T}  \right).
\]
\end{thm}
\textbf{Remark:} It is worth noting that we achieve the high probability bound without making any modifications to Algorithm~\ref{alg:3}. This is in contrast to the epoch gradient descent algorithm~\citep{COLT:Hazan:2011} that needs to shrink the domain size in order to obtain the desirable high probability bound, which could potentially lead to an additional computational cost in performing projection. We remove the shrinking step by effectively exploring the peeling technique~\citep{Local_RC}.

The number of projections required by Algorithm~\ref{alg:3}, according to Theorem~\ref{thm:2}, exhibits a linear dependence on the conditional number $L/\lambda$, which can be very large when dealing with ill-conditioned optimization problems. In the deterministic setting, the convergence rate only depends on the square root of the conditional number~\citep{nesterov2004introductory,Nesterov_Composite}. Thus, we conjecture that it may be possible to improve the dependence on the conditional number to its square root in the stochastic setting, a problem that will be examined in the future.

\section{Analysis}
We here present the proofs of main theorems. The omitted proofs are provided in the supplementary material.
\subsection{Proof of Theorem~\ref{thm:1}}
Since we make use of the the multi-stage learning strategy, the proof provided below is similar to the proof in~\citep{COLT:Hazan:2011}. We begin by analyzing the property of the inner loop in Algorithm~\ref{alg:3}, which is a combination of mini-batches and the extra-gradient descent. To this end, we have the following lemma.
\begin{lemma} \label{lem:inner:loop}
Let $\eta = 1/[\sqrt{6}L]$ in Algorithm~\ref{alg:3}. Then, we have
\begin{eqnarray}
F\left( \frac{1}{M}  \sum_{t=1}^M \z_t^k\right) -F(\w_*)& \leq & \frac{\|\w_{1}^k-\w_*\|^2}{2 M \eta} - \frac{\lambda}{2M} \sum_{t=1}^M \|\z_t^k - \w_*\|^2 \nonumber \\
 &  &  +  \frac{3 \eta}{M} \left( \sum_{t=1}^M \|\gb_t^k - \g_t^k\|^2 +     \sum_{t=1}^M \|\fb_t^k - \f_t^k\|^2  \right) \label{eqn:variance} \\
 & &  +  \frac{1}{M} \sum_{t=1}^M \langle \f_t^k-\fb_t^k, \z_t^k - \w_*  \rangle   \label{eqn:martingle}
\end{eqnarray}
where
\[
\g_t^k= \nabla F(\w_t^k) \textrm{ and } \f_t^k =\nabla F(\z_t^k).
\]
Taking the conditional expectation of the inequality, we have
\[
\E_{k-1}\left[ F\left( \frac{1}{M}  \sum_{t=1}^M \z_t^k\right) \right] -F(\w_*) \leq  \frac{\|\w_{1}^k-\w_*\|^2}{2 M \eta} + \frac{6 \eta G^2 }{B^k}.
\]
where $\E_{k-1}[\cdot]$ denotes the expectation conditioned on all the randomness up to epoch $k-1$.
\end{lemma}
The quantity in (\ref{eqn:variance}) illustrates the advantage of the extra-gradient descent, i.e., it is able to produce variance-dependent upper bound when applied to stochastic optimization. Because of mini-batches, the expectations of $\|\gb_t^k - \g_t^k\|^2$  and $\|\fb_t^k - \f_t^k\|^2$ are smaller than $G^2/B^k$, which leads to the tight upper bound in the second inequality.

Based on Lemma~\ref{lem:inner:loop}, we get the following lemma that bounds the expected excess risk in each epoch.
\begin{lemma} \label{lem:outer:loop}
Define
\[
\Delta_k =F(\w_1^k)-F(\w_*) .
\]
Set the parameters $\eta = 1/[\sqrt{6}L]$, $M=4/[\eta \lambda]$ and $B^1=12 \eta \lambda$ in Algorithm~\ref{alg:3}. For any $k$, we have
\[
\E[\Delta_k] \leq V_k = \frac{G^2}{\lambda 2^{k-2} }.
\]
\end{lemma}
\begin{proof}
It is straightforward to check that
\begin{equation} \label{eqn:bk}
B^k= 12 \eta \lambda  2^{k-1} =  \frac{ 24 \eta  G^2}{ V_k}.
\end{equation}

We prove this lemma by induction on $k$. When $k=1$, we know that
\[
\Delta_1 = F(\w_1^1) - F(\w_*) \overset{\text{(\ref{eqn:1})}}{\leq} \frac{2G^2}{\lambda } = \frac{G^2}{\lambda  2^{1-2}}=V_1.
\]
Assume that $\E[\Delta_k] \leq V_k$ for some $k \geq 1$, and we prove the inequality for $k+1$. From Lemma~\ref{lem:inner:loop}, we have
\[
 \E_{k-1}\left[ F\left(\w_1^{k+1}\right) \right] -F(\w_*) \leq  \frac{\|\w_{1}^k-\w_*\|^2}{2 M \eta} + \frac{6 \eta G^2 }{B^k}.
\]
Thus
\begin{eqnarray*}
& & \E \left[ F\left(\w_1^{k+1}\right) \right] -F(\w_*) \\
&\leq  &  \frac{\E[ \|\w_{1}^k-\w_*\|^2]}{2 M \eta} + \frac{6 \eta G^2 }{B^k} \\
&\overset{\text{(\ref{eqn:1})}}{\leq} & \frac{\E[ 2 (F(\w_{1}^k)- F(\w_*))/\lambda]}{2 M \eta} + \frac{6 \eta G^2 }{B^k} \\
&\overset{\text{(\ref{eqn:bk})}}{=}  & \frac{\E[ \Delta_k ]}{ M \eta \lambda} + \frac{ V_k}{4} \leq \frac{V_k}{ 4} + \frac{V_k}{4 } =V_{k+1}.
\end{eqnarray*}
\end{proof}

We are now at the position to prove Theorem~\ref{thm:1}.
\begin{proof}[Proof of Theorem~\ref{thm:1}]
From the stopping criterion of the outer loop in Algorithm~\ref{alg:3}, we know that the number of the epochs is given by the largest value of $k$ such that
\[
2 M \sum_{i=1}^k B^i \leq T.
\]
Since
\[
2 M \sum_{i=1}^k B^i  = 24 M \eta \lambda \sum_{i=1}^k 2^{i-1}=96 (2^k-1),
\]
the final epoch is given by
\[
k^\dag= \left \lfloor \log_2 \left( \frac{T}{96}+1\right) \right  \rfloor,
\]
and the final output is $\w_1^{k^\dag+1}$. From Lemma~\ref{lem:outer:loop}, we have
\[
\E[F(\w_1^{k^\dag+1})]-F(\w_*) \leq  V_{k^\dag+1} = \frac{G^2}{2^{k^\dag-1}\lambda} \leq \frac{384  G^2}{ \lambda T},
\]
where we use the fact
\[
2^{k^\dag} \geq \frac{1}{2} \left(\frac{ T}{96}+1 \right)\geq \frac{ T}{192}.
\]
The total number of projections is
\[
2 M k^\dag = \frac{8 \sqrt{6} L}{\lambda} \left \lfloor \log_2\left( \frac{T}{96}+1\right) \right \rfloor.
\]

\end{proof}
\subsection{Proof of Theorem~\ref{thm:2}}
Compared to the proof of Theorem~\ref{thm:1}, the main difference here is that we need a high probability version of Lemma~\ref{lem:inner:loop}. Specifically, we need to provide high probability bounds for the quantities in (\ref{eqn:variance}) and (\ref{eqn:martingle}).

To bound the variances given in (\ref{eqn:variance}), we need the following norm concentration inequality in Hilbert Space \citep{smale-2009-geometry}.
\begin{lemma} \label{lem:norm:hilbert}
Let $\H$ be a Hilbert Space and let $\xi$ be a random variable on $(\Z, \rho)$ with values in $\H$. Assume $\|\xi\| \leq B < \infty$ almost surely. Let $\{\xi_i\}_{i=1}^m$ be independent random drawers of $\rho$. For any $0 < \delta <1$, with a probability at least $1-\delta$,
\[
\left \| \frac{1}{m} \sum_{i=1}^m (\xi_i - \E [\xi_i] ) \right \| \leq  \frac{4B}{\sqrt{m}} \log \frac{2}{\delta}.
\]
\end{lemma}
Based on Lemma~\ref{lem:norm:hilbert}, it is straightforward to prove the following lemma.
\begin{lemma} \label{lem:variance}
With a probability at least $1-\tilde{\delta}/2$, we have
\begin{equation} \label{eqn:smooth:grad:1}
\|\gb_t^k-  \g_t^k \| \leq \frac{4G}{\sqrt{B^k}} \log \frac{4 M}{\tilde{\delta}}, \ \forall \ t=1,\ldots, M.
\end{equation}
Similarly, with a probability at least $1-\tilde{\delta}/4$, we have
\begin{equation} \label{eqn:smooth:grad:2}
\|\fb_t^k -  \f_t^k \| \leq \frac{4G}{\sqrt{B^k}} \log \frac{8 M}{\tilde{\delta}}, \ \forall \ t=1,\ldots, M.
\end{equation}
\end{lemma}

We define the Martingale difference sequence:
\[
Z_t^k= \langle \f_t^k-\fb_t^k, \z_t^k - \w_*  \rangle.
\]
In order to bound the summation of $Z_t^k$ in~(\ref{eqn:martingle}), we make use of the Berstein inequality for martingales~\citep{bianchi-2006-prediction} and the peeling technique described in~\citep{Local_RC}, leading to the following Lemma.
\begin{lemma}\label{lem:3}
We use $E_1$ to denote the event that all the inequalities in (\ref{eqn:smooth:grad:2}) hold. On event $E_1$, with a probability at least $1 - \tilde{\delta}/4$, we have
\begin{eqnarray*}
\sum_{t=1}^M Z_t^k \leq  \frac{4G^2 \eta M}{B^k} \log^2 \frac{8 M}{\tilde{\delta}}  + \frac{G^2}{\lambda B^k } \left[ 1  +  64 \log^2 \frac{8 M}{\tilde{\delta}}  \left( \log \frac{4n}{\tilde{\delta}} + \frac{4}{9} \log^2 \frac{4n}{\tilde{\delta}} \right)\right]    +\frac{\lambda}{2} \sum_{t=1}^M \|\z_t^k - \w_*\|^2,
\end{eqnarray*}
where
\begin{equation} \label{eqn:lemma:n}
n = \left \lceil \log_2 \frac{4MB^k }{\eta \lambda} \right \rceil.
\end{equation}
\end{lemma}

Substituting the results in Lemmas~\ref{lem:variance} and~\ref{lem:3} into Lemma~\ref{lem:inner:loop}, we obtain the lemma below.
\begin{lemma} \label{lem:10:smooth:2}
For any $0 < \tilde{\delta} < 1$, with a probability at least $1-\tilde{\delta}$, we have
\[
\begin{split}
F\left( \frac{1}{M}  \sum_{t=1}^M \z_t^k\right) -F(\w_*) \leq  &\frac{\|\w_{1}^k-\w_*\|^2}{2 M \eta}  +  \frac{100G^2\eta }{B^k} \log^2 \frac{8 M}{\tilde{\delta}}  \\
 & + \frac{G^2}{ \lambda B^k M} \left[ 1 + 64 \log^2 \frac{8 M}{\tilde{\delta}}  \left( \log \frac{4n}{\tilde{\delta}} + \frac{4}{9} \log^2 \frac{4n}{\tilde{\delta}} \right) \right],
\end{split}
\]
where $n$ is given in~(\ref{eqn:lemma:n}).
\end{lemma}
Based on Lemma~\ref{lem:10:smooth:2}, we provide a high probability version of Lemma~\ref{lem:outer:loop}, that bounds the excess risk in each epoch with a high probability.
\begin{lemma} \label{lem:11:smooth}
Set the parameters $\eta = 1/[\sqrt{6}L]$, $M=4/[\eta \lambda]$ and $B^1=\alpha \eta \lambda$ in Algorithm~\ref{alg:3}, where $\alpha$  is defined in (\ref{eqn:smooth:alpha:2}). For any $k$, with a probability at least $(1-\tilde{\delta})^{k-1}$, we have
\[
\Delta_k =F(\w_1^k)-F(\w_*) \leq V_k = \frac{G^2}{\lambda 2^{k-2} }.
\]
\end{lemma}
\begin{proof}
We follow the logic used in the proof of Lemma~\ref{lem:outer:loop}.

It is straightforward to check that
\[
B^k= \alpha \eta \lambda  2^{k-1} =  \frac{ 2 \alpha \eta  G^2}{ V_k}.
\]

When $k=1$, with a probability $(1-\tilde{\delta})^{1-1}=1$, we have
\[
\Delta_1 = F(\w_1^1) - F(\w_*) \overset{\text{(\ref{eqn:1})}}{\leq}  \frac{2G^2}{\lambda } = \frac{G^2}{\lambda  2^{1-2}}=V_1.
\]
Assume that with a probability at least $(1-\tilde{\delta})^{k-1}$, $\Delta_k \leq V_k$ for some $k \geq 1$. We now prove the case for $k+1$. Notice that $N$ defined in (\ref{eqn:lemma:N}) is larger than $n$ defined in~(\ref{eqn:lemma:n}). From Lemma~\ref{lem:10:smooth:2}, with a probability at least $1-\tilde{\delta}$, we have
\[
\begin{split}
& \Delta_{k+1} =  F(\w_1^{k+1})-F(\w_*) \\
\leq &  \frac{\|\w_1^k-\w_*\|^2}{2 M \eta} +  \frac{100G^2\eta  }{B^k} \log^2 \frac{8 M}{\tilde{\delta}}  + \frac{G^2}{ \lambda B^k M} \left[ 1 + 64 \log^2 \frac{8 M}{\tilde{\delta}}  \left( \log \frac{4N}{\tilde{\delta}} + \frac{4}{9} \log^2 \frac{4N}{\tilde{\delta}} \right) \right]\\
\leq &  \frac{\Delta_k}{4 } +   \frac{400}{\alpha} \log^2 \frac{8 M}{\tilde{\delta}} \frac{V_k}{8}  +  \frac{1}{  \alpha } \left[ 1 + 64 \log^2 \frac{8 M}{\tilde{\delta}}  \left( \log \frac{4N}{\tilde{\delta}} + \frac{4}{9} \log^2 \frac{4N}{\tilde{\delta}} \right) \right] \frac{V_k}{  8  }. \\
\end{split}
\]

Using the definition of $\alpha$ in (\ref{eqn:smooth:alpha:2}), with a probability at least $(1-\tilde{\delta})^k$ we have,
\[
\Delta_{k+1}  \leq \frac{1}{4} V_k  + \frac{1}{8} V_k  + \frac{1}{8} V_k = \frac{1}{2} V_k = V_{k+1}.
\]
\end{proof}

Now, we provide the proof of Theorem~\ref{thm:2}.
\begin{proof}[Proof of Theorem~\ref{thm:2}]
The number of epochs made is given by the largest value of $k$ satisfying $2 M \sum_{i=1}^k B^i \leq T$. Since
\[
2 M \sum_{i=1}^k B^i =2 M \alpha \lambda \eta \sum_{i=1}^k  2^{i-1} =8 \alpha(2^k-1),
\]
$k^\dag$ defined in (\ref{eqn:smooth:k}) is the value of the final epoch, and the final output is $\w_1^{k^\dag+1}$. From Lemma~\ref{lem:11:smooth}, we have with a probability at least $(1-\tilde{\delta})^{k^\dag}$
\[
\begin{split}
& F(\w_1^{k^\dag+1})-F(\w_*) = \Delta_{k^\dag+1} \leq  V_{k^\dag+1} = \frac{G^2}{2^{k^\dag-1}\lambda} = \frac{2G^2}{2^{k^\dag}\lambda} \leq \frac{32 \alpha G^2}{ \lambda T},
\end{split}
\]
where we use the fact
\[
2^{k^\dag} \geq \frac{1}{2} \left(\frac{ T}{8\alpha}+1 \right)\geq \frac{ T}{16 \alpha}.
\]
We complete the proof by using the property that $(1-\frac{1}{x})^x$ is an increasing function when $x>1$, which implies
\[
\begin{split}
& (1-\tilde{\delta})^{k^\dag}=\left(1-\frac{\delta}{k^\dag}\right)^{k^\dag}=\left(\left(1-\frac{1}{k^\dag/\delta}\right)^{k^\dag/\delta}
\right)^{\delta} \geq  \left(\left(1-\frac{1}{1/\delta}\right)^{1/\delta}\right)^{\delta}=1-\delta.
\end{split}
\]
\end{proof}
\section{Experiments}
In this section, we present numerical experiments to support our theoretical analysis. We studied the following algorithms:
\begin{compactenum}
  \item $\log T$: the proposed algorithm that is optimal for SOS$^2$C but only needs $\log(T)$ projections;
  \item EP\_GD: the epoch gradient descent developed in~\citep{COLT:Hazan:2011}, which is also optimal for SOS$^2$C but needs $O(T)$ projections;
  \item SGD: the stochastic gradient descent with step size $\eta_t=1/(\lambda t)$~\citep{Shalev-ShwartzSSC11}, which achieves $O(\log T/T)$ rate of convergence for general SOS$^2$C and needs $O(T)$ projections.
\end{compactenum}
We first consider the a simple stochastic optimization problem adapted from~\citep{ICML2012Rakhlin}, which is both smooth and strongly convex. The objective function is $F(W)=\frac{1}{2} \|W\|_F^2$ and the domain is the $5 \times 5$ dimensional positive semidefinite (PSD) cone. The stochastic gradient oracle, given a point $W$, returns the stochastic gradient $W+Z$ where $Z$ is uniformly distributed in $[-1,1]^{5\times 5}$. Because of the noise matrix $Z$, all the immediate solutions are not PDS and we need to project them back to the PSD cone. To ensure the eigendecomposition only involving real numbers, we further require $Z$ to be symmetric. Notice that for this problem we know $W_* = \argmin_{W \succeq o} F(W)=0^{5\times 5}$. Since the gradient of $W_*$ is $0^{5\times 5}$, it can be shown that SGD also achieves the optimal $O(1/T)$ rate of convergence on this problem~\citep{ICML2012Rakhlin}.

Let $W_T$ be the solution returned after making $T$ calls to the gradient oracle. To verify if the proposed algorithm achieves an $O(1/T)$ convergence, we measure $(F(W_T)-F(W_*))\times T$ versus $T$, which is given in Fig.~\ref{fig:1:a}. We observe that when $T$ is sufficiently large, quantity $(F(W_T)-F(W_*))\times T$ essentially becomes a constant for all three algorithms, implying $O(1/T)$ convergence rates for all the algorithms. We also observe that the constant achieved by the proposed algorithm is slightly larger than the two competitors, which can be attributed to the term $(\log\log T)^4$ in our bound in Theorem~\ref{thm:2}. To demonstrate the advantage of our algorithm, we plot the value of the objective function versus the number of projections $P$ in Fig.~\ref{fig:1:b}. We observe that using our algorithm, the objective function is reduced significantly faster than other algorithms w.r.t.~the number of projections.
\begin{figure*}[t]
\centering
  \subfigure[$(F(\w_T)-F(\w_*))\times T$ versus $T$]{
    \label{fig:1:a} %% label for first subfigure
    \includegraphics[width=0.45\textwidth]{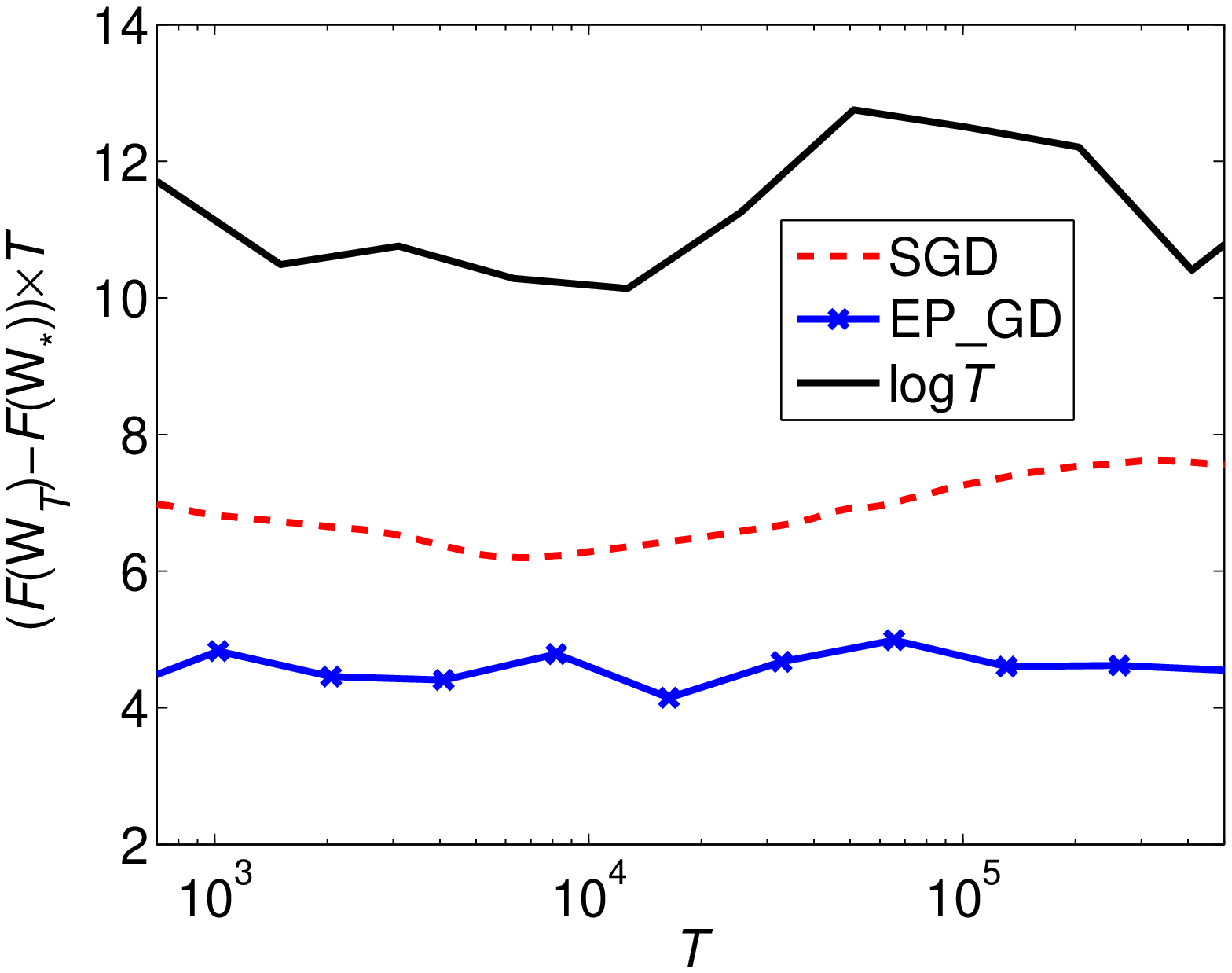}}
\subfigure[$F(\w_T)$ versus the number of projections $P$]{
    \label{fig:1:b} %% label for second subfigure
    \includegraphics[width=0.45\textwidth]{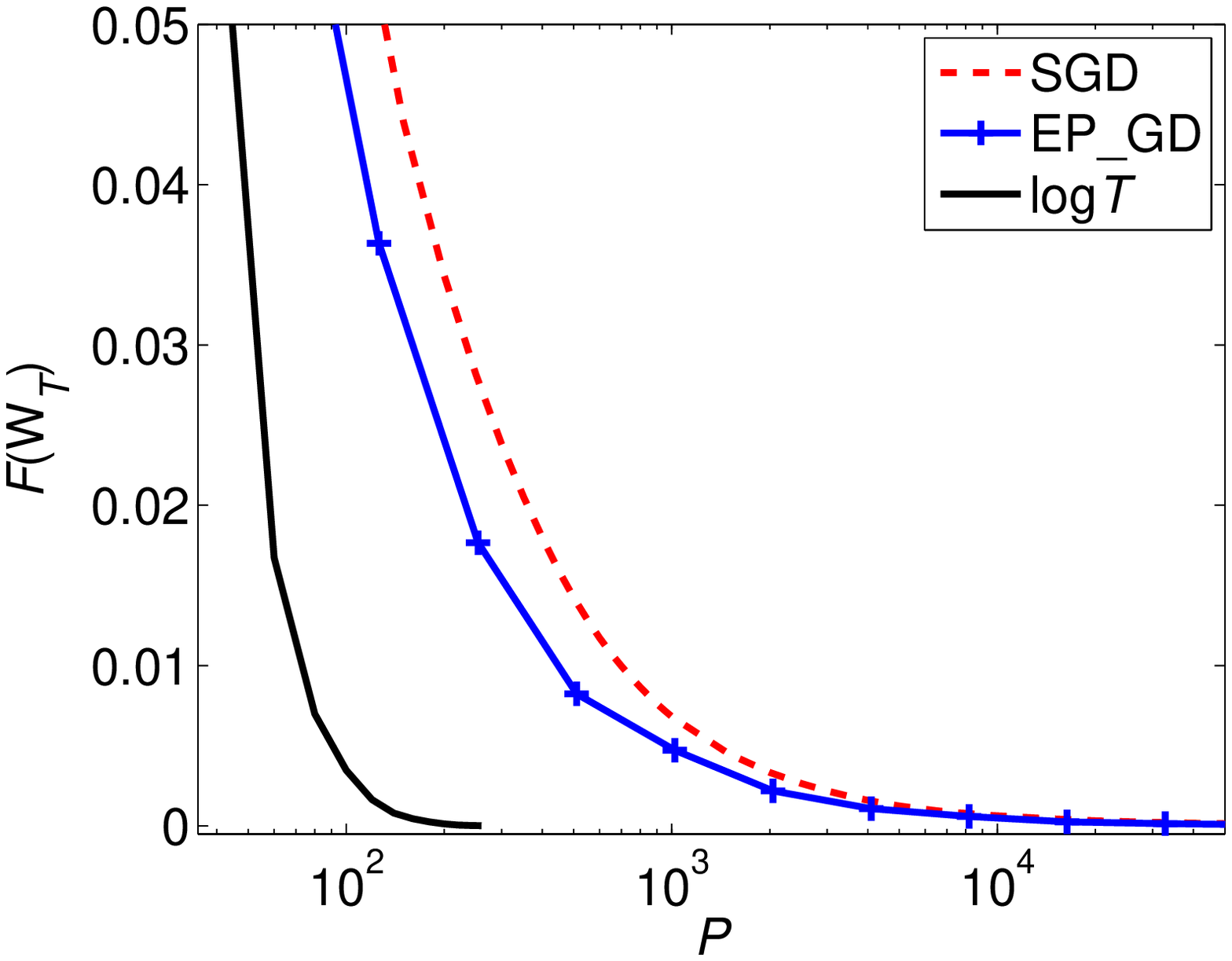}}
  \caption{Results for stochastic optimization of $F(W) = \frac{1}{2}\|W\|_F^2$ over the PSD cone. The experiments are repeated 10 times and the averages are reported.}
  \label{fig:1}
\end{figure*}

In the second experiment, we apply our algorithm to the regularized distance metric learning~\citep{NIPS2009_RDM}. The goal is to solve the following problem
\[
\min_{W \succeq 0} \E_{(\x_i,y_i),(\x_j,y_j)} [\ell( y_{ij} (1-\|\x_i-\x_j\|^2_M))] + \frac{\lambda}{2} \|W\|_F^2,
\]
where $\x_i$ is the instance, and $y_i$ is $\x_i$'s label, $y_{ij}$ is derived from labels $y_i$ and $y_j$ (i.e., $y_{ij}=1$ if $y_i = y_j$ and $-1$ otherwise), $\|\x\|^2_M=\x^\top M \x$, and $\ell(z)=\log(1 + \exp(-z))$ is the logit loss.
%It is easy to verify that the objective function is $\lambda+4$-smooth and $\lambda$-strongly convex if $\|\x\| \leq 1$. We set $\lambda=0.1$ and test our algorithm on the Mushrooms and Adult data sets~\citep{LibSVM}.
During the optimization process, the call to the gradient oracle corresponds to generate a training pair $\{(\x_i,y_i),(\x_j,y_j)\}$ randomly. To estimate the value of objective function, we evaluate the average empirical loss on $10^4$ testing pairs, which are also generated randomly. Fig.~\ref{fig:2} shows the value of the objective function versus the number of projections $P$. Again, this result validates that the proposed algorithm $\log T$ is able to reduce the number of projections dramatically without hurting the performance.
\begin{figure*}[t]
\centering
  \subfigure[Mushrooms]{
    \label{fig:2:a} %% label for first subfigure
    \includegraphics[width=0.45\textwidth]{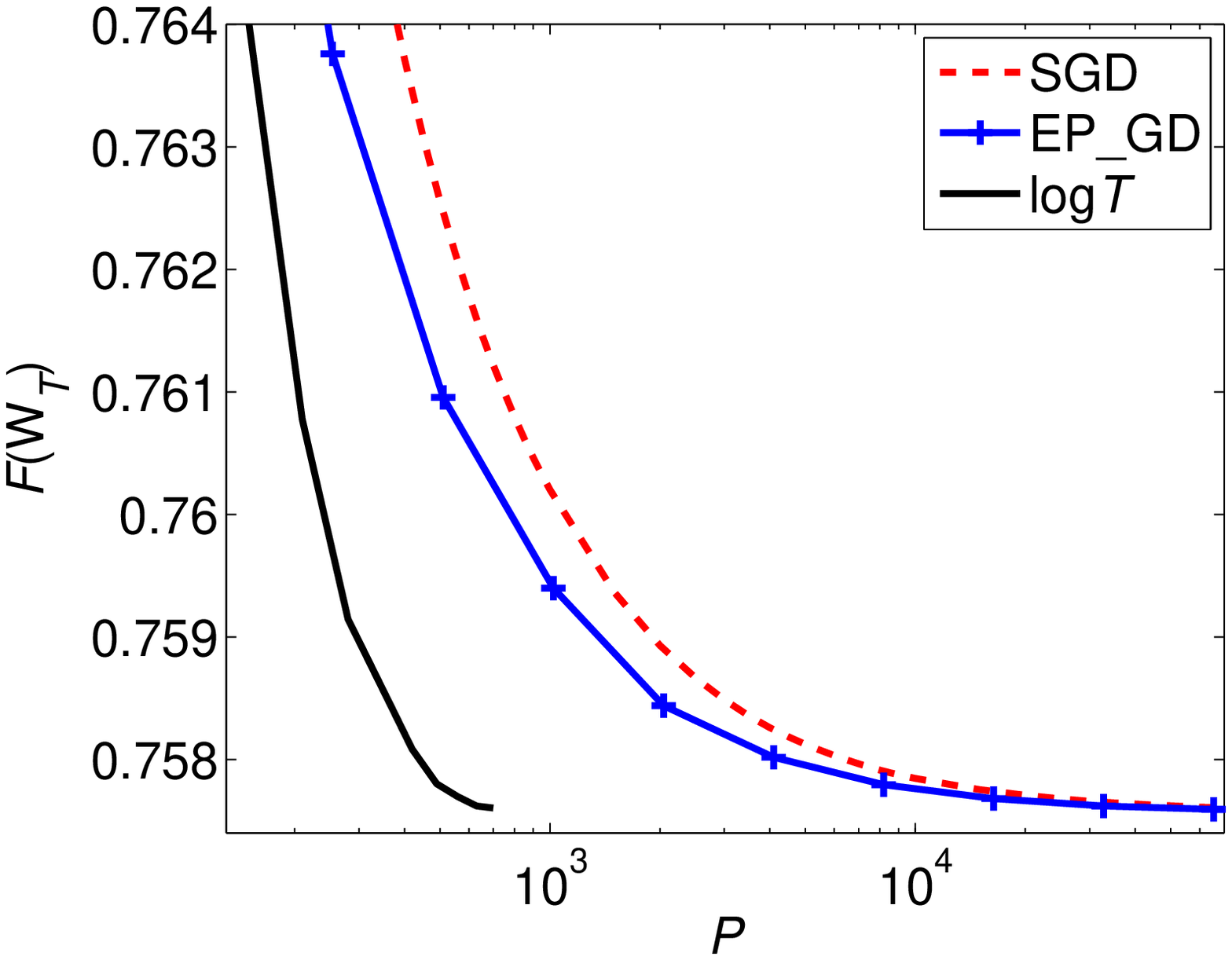}}
\subfigure[Adult]{
    \label{fig:2:b} %% label for second subfigure
    \includegraphics[width=0.45\textwidth]{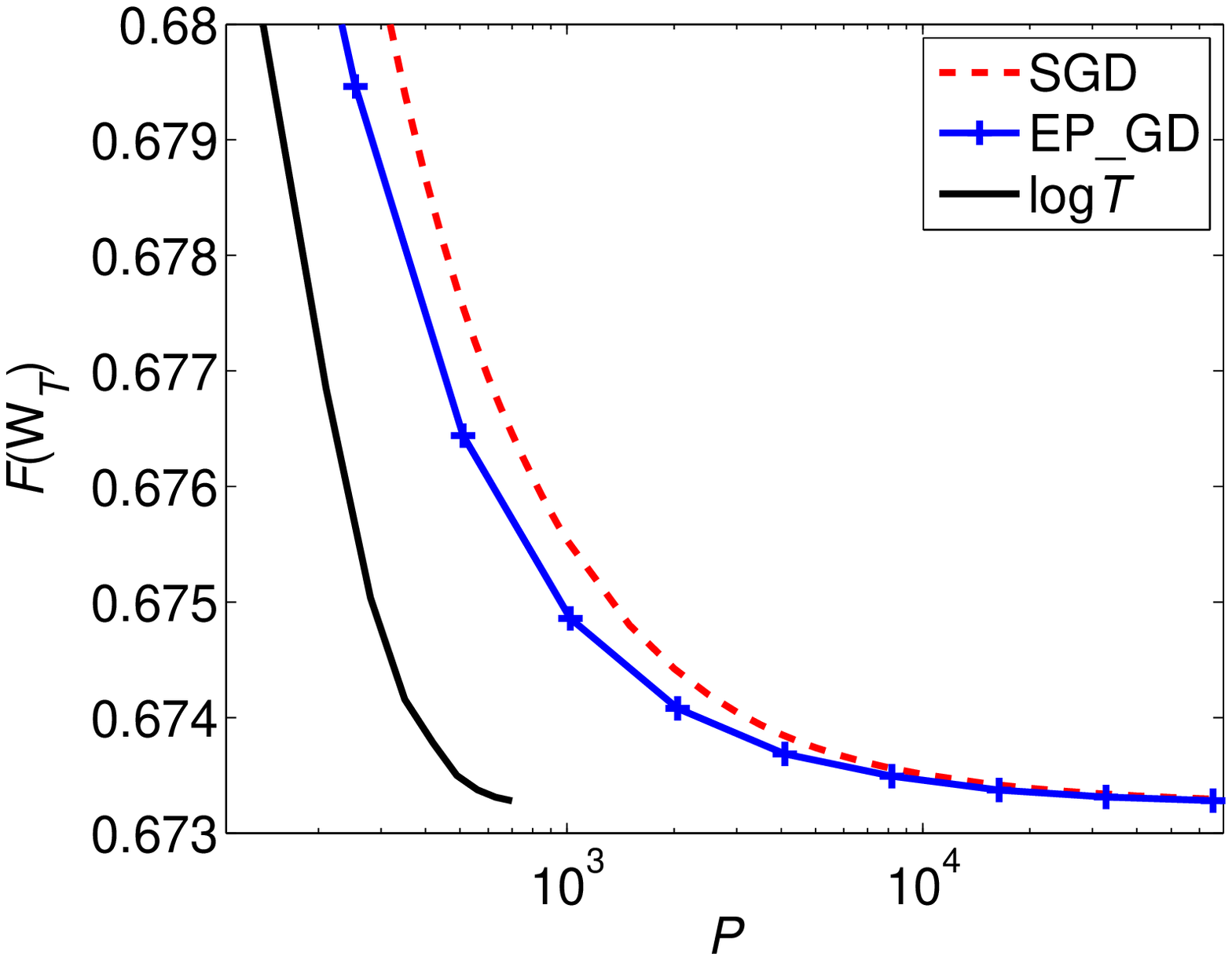}}
\caption{Results for the regularized distance metric learning on the Mushrooms and Adult data sets. $F(W_T)$ is measured on $10^4$ testing pairs and the horizontal axis $P$ measures the number of projections performed by each algorithm. The experiments are repeated 10 times and the averages are reported.}
  \label{fig:2}
\end{figure*}
\section{Conclusion}
In this paper, we study the problem of reducing the number of projections in stochastic optimization by exploring the property of smoothness. When the target function is smooth and strongly convex, we propose a novel algorithm that achieves the optimal $O(1/T)$ rate of convergence by only performing $O(\log T)$ projections.

An open question is how to extend our results to stochastic composite optimization~\citep{Lan:SCO}, where the objective function is a combination of non-smooth and smooth stochastic components. We plan to explore the composite gradient mapping technique, introduced in~\citep{Nesterov_Composite}, to see if we can achieve an $O(1/T)$ convergence rate with only $O(\log T)$ projections.

\bibliography{E:/MyPaper/ref}

\appendix
\section{Proof of Lemma~\ref{lem:inner:loop}}
We need the following lemma that characterizes the property of the extra-gradient descent.
\begin{lemma}[Lemma~3.1 in \citep{nemirovski-2005-prox}] \label{lem:nemirovski}
Let $\Z$ be a convex compact set in Euclidean space $\mathcal{E}$ with inner product $\langle \cdot, \cdot \rangle$, let $\| \cdot \|$ be a norm on $\mathcal{E}$ and $\|\cdot\|_*$ be its dual norm, and let $\omega(\z): \Z \mapsto \R$ be a $\alpha$-strongly convex function with respect to $\|\cdot\|$. The Bregman distance associated with $\omega$ for points $\z, \w \in \Z$ is defined as
\[
B_\omega(\z,\w)=\omega(\z)-\omega(\w)-\langle \z-\w, \nabla \omega(\w) \rangle.
\]
Let $\U$ be a convex and closed subset of $\Z$, and let $\z_- \in \Z$, let $\bm \xi, \bm \eta \in \mathcal{E}$, and let $\gamma >0$.  Consider the points
\begin{eqnarray*}
\w=\argmin_{\y \in \U} \{\langle \gamma \bm \xi - \nabla \omega(\z_-), \y \rangle + \omega(\y)\},\\
\z_+=\argmin_{\y \in \U} \{\langle \gamma \bm \eta - \nabla \omega(\z_-), \y \rangle + \omega(\y)\}.
\end{eqnarray*}
Then for all $\z \in \U$ one has
\[
 \langle \w - \z,   \gamma \bm \eta \rangle \leq  B_\omega(\z, \z_-)-B_\omega(\z,\z_+) + \frac{\gamma^2}{\alpha} \|\bm \eta -\bm \xi \|_*^2 - \frac{\alpha}{2} \{ \|\w - \z_- \|^2 + \| \z_+ - \w \|^2 \}.
\]
\end{lemma}
\begin{proof}[Proof of Lemma~\ref{lem:inner:loop}]
We first state the inner loop in Algorithm~\ref{alg:3} below.
\begin{algorithmic}
\FOR{$t=1$ to $M$}
\STATE Compute the average gradient at $\w_t^k$ over $B^k$ calls to the gradient oracle
\[
\gb_t^k = \frac{1}{B^k} \sum_{i=1}^{B^k} \gh(\w_t^k,i)
\]
\STATE Update
\[
\z_t^k=\Pi_{\D}\left(\w_t^k-\eta \gb_t^k  \right)
\]
\STATE Compute the average gradient at $\z_t^k$ over $B^k$ calls to the gradient oracle
\[
\fb_t^k = \frac{1}{B^k} \sum_{i=1}^{B^k}\gh(\z_t^k,i)
\]
\STATE Update
\[
\w_{t+1}^k=\Pi_{\D}\left(\w_t^k-\eta \fb_t^k  \right)
\]
\ENDFOR
\end{algorithmic}

To simplify the notation, we define
\[
\g_t^k= \nabla F(\w_t^k) \textrm{ and } \f_t^k =\nabla F(\z_t^k).
\]
Let the two norms $\|\cdot\|$ and $\|\cdot\|_*$ in Lemma~\ref{lem:nemirovski} be the vector $\ell_2$ norm. Each iteration in the inner loop satisfies the conditions in Lemma~\ref{lem:nemirovski} by doing the mappings below:
\begin{gather*}
\U =\Z=\mathcal{E}\leftarrow \D,  \  \omega(\z) \leftarrow \frac{1}{2} \|\z\|^2, \ \alpha \leftarrow 1, \ \gamma \leftarrow \eta, \\
\z_- \leftarrow \w_t^k, \ \bm \xi \leftarrow \gb_t^k, \ \bm \eta  \leftarrow \fb_t^k, \ \w \leftarrow \z_t^k, \ \z_+ \leftarrow \w_{t+1}^k, \ \z \leftarrow \w_*.
\end{gather*}
Following  Lemma~\ref{lem:nemirovski}, we have
\begin{equation} \label{eqn:extra}
\begin{split}
& \langle \z_t^k -\w_*, \eta \fb_t^k  \rangle \\
\leq & \frac{\|\w_{t}^k-\w_*\|^2}{2} -\frac{\|\w_{t+1}^k-\w_*\|^2}{2} + \eta^2 \|\gb_t^k - \fb_t^k\|^2 - \frac{1}{2} \| \w_t^k - \z_t^k\|^2 \\
\leq & \frac{\|\w_{t}^k-\w_*\|^2}{2} -\frac{\|\w_{t+1}^k-\w_*\|^2}{2} + 3 \eta^2 \left( \|\gb_t^k - \g_t^k\|^2 + \|\fb_t^k - \f_t^k\|^2 + \|\g_t^k - \f_t^k\|^2 \right)\\
& - \frac{1}{2} \| \w_t^k - \z_t^k\|^2 \\
\leq & \frac{\|\w_{t}^k-\w_*\|^2}{2} -\frac{\|\w_{t+1}^k-\w_*\|^2}{2} + 3 \eta^2 \left( \|\gb_t^k - \g_t^k\|^2 + \|\fb_t^k - \f_t^k\|^2  \right) \\
& + 3 \eta^2 \|\g_t^k -\f_t^k \|^2 - \frac{1}{2} \| \w_t^k - \z_t^k\|^2 \\
\leq & \frac{\|\w_{t}^k-\w_*\|^2}{2} -\frac{\|\w_{t+1}^k-\w_*\|^2}{2} + 3 \eta^2 \left( \|\gb_t^k - \g_t^k\|^2 + \|\fb_t^k - \f_t^k\|^2  \right) \\
& +  3 \eta^2 L^2 \|\w_t^k -\z_t^k \|^2 - \frac{1}{2} \| \w_t^k - \z_t^k\|^2 \\
\leq & \frac{\|\w_{t}^k-\w_*\|^2}{2} -\frac{\|\w_{t+1}^k-\w_*\|^2}{2} + 3 \eta^2 \left( \|\gb_t^k - \g_t^k\|^2 + \|\fb_t^k - \f_t^k\|^2  \right),
\end{split}
\end{equation}
where in the fifth line we use the smoothness assumption
\[
\|\g_t^k -\f_t^k \|= \|\nabla F(\w_t^k) - \nabla F(\z_t^k) \| \leq L \|\w_t^k-\z_t^k\|.
\]

From the property of $\lambda$-strongly convex function and (\ref{eqn:extra}), we obtain
\[
\begin{split}
& F(\z_t^k) - F(\w_*) \\
\leq & \langle \f_t^k, \z_t^k - \w_*  \rangle - \frac{\lambda}{2} \|\z_t^k - \w_*\|^2 \\
= & \langle \fb_t^k, \z_t^k - \w_*  \rangle +  \langle \f_t^k-\fb_t^k, \z_t^k - \w_*  \rangle - \frac{\lambda}{2} \|\z_t^k - \w_*\|^2 \\
\leq & \frac{\|\w_t^k-\w_*\|^2}{2 \eta} -\frac{\|\w_{t+1}^k-\w_*\|^2}{2\eta} + 3 \eta \left( \|\gb_t^k - \g_t^k\|^2 + \|\fb_t^k - \f_t^k\|^2   \right) \\
& +  \langle \f_t^k-\fb_t^k, \z_t^k - \w_*  \rangle - \frac{\lambda}{2} \|\z_t^k - \w_*\|^2.
\end{split}
\]
Summing up over all $t=1,2,\ldots,M$, we have
\[
\begin{split}
& \sum_{t=1}^M F(\z_t^k) - M F(\w_*) \\
\leq & \frac{\|\w_{1}^k-\w_*\|^2}{2 \eta}  + 3 \eta \left(\sum_{t=1}^M \|\gb_t^k - \g_t^k\|^2 +     \sum_{t=1}^M \|\fb_t^k - \f_t^k\|^2 \right) \\
& +  \sum_{t=1}^M \langle \f_t^k-\fb_t^k, \z_t^k - \w_*  \rangle   - \frac{\lambda}{2} \sum_{t=1}^M \|\z_t^k - \w_*\|^2 . \\
\end{split}
\]
Dividing both sides by $M$ and following Jensen's inequality, we have
\begin{equation}\label{eqn:inner:bound1}
\begin{split}
 & F\left( \frac{1}{M}  \sum_{t=1}^M \z_t^k\right) -F(\w_*) \\
 \leq & \frac{1}{M}\sum_{t=1}^M  F( \z_t^k) -F(\w_*)\\
 \leq & \frac{\|\w_{1}^k-\w_*\|^2}{2 M \eta}  +  \frac{3 \eta}{M} \left( \sum_{t=1}^M \|\gb_t^k - \g_t^k\|^2 +     \sum_{t=1}^M \|\fb_t^k - \f_t^k\|^2  \right) +  \\
 & \frac{1}{M} \sum_{t=1}^M \langle \f_t^k-\fb_t^k, \z_t^k - \w_*  \rangle  - \frac{\lambda}{2M} \sum_{t=1}^M \|\z_t^k - \w_*\|^2.
\end{split}
\end{equation}
which gives the first inequality in Lemma~\ref{lem:inner:loop}.

Let $\E_{k-1}[\cdot]$ denote the expectation conditioned on all the randomness up to epoch $k-1$ and $\E_{k}^{t-1} [\cdot]$ denote the expectation conditioned on all the randomness up to the $t-1$-th iteration in the $k$-th epoch. Taking the conditional expectation of (\ref{eqn:inner:bound1}), we have
\begin{equation} \label{eqn:conditional}
\begin{split}
 & \E_{k-1}\left[F\left( \frac{1}{M}  \sum_{t=1}^M \z_t^k\right)\right] -F(\w_*) \\
 \leq & \frac{\|\w_{1}^k-\w_*\|^2}{2 M \eta}  +  \frac{3 \eta}{M} \left( \sum_{t=1}^M \E_{k-1}\left[\|\gb_t^k - \g_t^k\|^2\right] +     \sum_{t=1}^M \E_{k-1}\left[\|\fb_t^k - \f_t^k\|^2\right]  \right) \\
 & +  \frac{1}{M} \sum_{t=1}^M \E_{k-1}\left[ \langle \f_t^k-\fb_t^k, \z_t^k - \w_*  \rangle\right],
\end{split}
\end{equation}
where we drop the last term, since it is negative. To bound $\E_{k-1}\left[\|\gb_t^k - \g_t^k\|^2\right]$, we have
\begin{equation} \label{eqn:expected:variance:1}
\begin{split}
& \E_{k-1}\left[\|\gb_t^k - \g_t^k\|^2\right]=\E_{k-1}\left[\left\|\frac{1}{B^k} \sum_{i=1}^{B^k} \gh(\w_t^k,i)- \g_t^k\right\|^2\right] \\
= & \E_{k-1}\left[\left\|\frac{1}{B^k} \sum_{i=1}^{B^k} \left(\gh(\w_t^k,i)- \g_t^k\right)\right\|^2\right] \\
= & \frac{1}{[B^k]^2} \sum_{i=1}^{B^k} \E_{k-1}\left[   \left\|\gh(\w_t^k,i)- \g_t^k\right\|^2 \right] \\
& + \frac{1}{[B^k]^2} \E_{k-1} \left[
\sum_{i \neq j} \left \langle \E_{k}^{t-1} \left[ \gh(\w_t^k,i)- \g_t^k \right], \E_{k}^{t-1}\left[  \gh(\w_t^k,j)- \g_t^k  \right] \right\rangle \right] \\
= & \frac{1}{[B^k]^2} \left( \sum_{i=1}^{B^k} \E_{k-1}\left[   \left\|\gh(\w_t^k,i)- \g_t^k\right\|^2 \right]  \right)  \leq \frac{G^2}{B^k},
\end{split}
\end{equation}
where we make use of the facts $\gh(\w_t^k,i)$ and $\gh(\w_t^k,j)$ are independent when $i \neq j$, and
\[
\E_{k}^{t-1} \left[ \gh(\w_t^k,i)- \g_t^k \right]=0 , \   \E_{k}^{t-1} \left[ \|\gh(\w_t^k,i)- \g_t^k \|^2\right] \leq \E_{k}^{t-1} \left[ \|\gh(\w_t^k,i)\|^2\right] \leq G^2, \ \forall i=1,\ldots,B^k.
\]
Similarly, we also have
\begin{equation}\label{eqn:expected:variance:2}
\E_{k-1}\left[\|\fb_t^k - \f_t^k\|^2\right]  \leq \frac{G^2}{B^k}.
\end{equation}
 Notice that $\fb_t^k$ is an unbiased estimate of $\f_t^k$, thus
\begin{equation} \label{eqn:expected:difference}
\E_{k-1}\left[ \langle \f_t^k-\fb_t^k, \z_t^k - \w_*  \rangle \right] = \E_{k-1}\left[ \langle \E_{k}^{t-1}\left[\f_t^k-\fb_t^k \right], \z_t^k - \w_*  \rangle \right] = 0.
\end{equation}
Substituting (\ref{eqn:expected:variance:1}), (\ref{eqn:expected:variance:2}), and (\ref{eqn:expected:difference}) into (\ref{eqn:conditional}), we get the second inequality in Lemma~\ref{lem:inner:loop}.
\end{proof}

\section{Proof of Lemma~\ref{lem:variance}}
\begin{proof}
Recall that $\gb_t^k = \frac{1}{B^k} \sum_{i=1}^{B^k} \gh(\w_t^k,i)$, thus
\[
\|\gb_t^k - \g_t^k \|=\left \| \frac{1}{B^k} \sum_{i=1}^{B^k} \gh(\w_t^k,i) -  \g_t^k \right \|.
\]
Since $\| \gh(\w_t^k,i)\| \leq G$, and $\E [\gh(\w_t^k,i)]= \g_t^k$, we have with a probability at least $1-\delta$
\[
 \|\gb_t^k - \g_t^k \| \leq   \frac{4G}{\sqrt{B^k}} \log \frac{2}{\delta}.
\]
We obtain (\ref{eqn:smooth:grad:1}) by the union bound and setting $\tilde{\delta}/2= M \delta$. The inequality in (\ref{eqn:smooth:grad:2}) can be proved in the same way.
\end{proof}
\section{Proof of Lemma~\ref{lem:3}} \label{sec:pro:lema3}
We first state the Berstein inequality for martingales~\citep{bianchi-2006-prediction}, which is used in the proof below.
\begin{thm} \label{thm:bernstein} (Bernstein's inequality for martingales). Let $X_1, \ldots , X_n$ be a bounded martingale difference sequence with respect to the filtration $\F = (\F_i)_{1\leq i\leq n}$ and with $|X_i| \leq K$. Let
\[
S_i = \sum_{j=1}^i X_j
\]
be the associated martingale. Denote the sum of the conditional variances by
\[
    \Sigma_n^2 = \sum_{t=1}^n \E\left[X_t^2|\F_{t-1}\right].
\]
Then for all constants $t$, $\nu > 0$,
\[
\Pr\left[ \max\limits_{i=1,\ldots,n} S_i > t \mbox{ and } \Sigma_n^2 \leq \nu \right] \leq \exp\left(-\frac{t^2}{2(\nu + Kt/3)} \right),
\]
and therefore,
\[
    \Pr\left[ \max\limits_{i=1,\ldots,n} S_i > \sqrt{2\nu t} + \frac{2}{3}Kt \mbox{ and } \Sigma_n^2 \leq \nu \right] \leq e^{-t}.
\]
\end{thm}

To simplify the notation, we define
\begin{eqnarray*}
A&=& \sum_{i=1}^M  \|\z_t^k - \w_*\|^2 \leq  \frac{4MG^2}{\lambda^2},\\
C&=&\frac{4G}{\sqrt{B^k}} \log \frac{8 M}{\tilde{\delta}}.
\end{eqnarray*}
In the analysis below, we consider two different scenarios, i.e., $A \leq \eta G^2/[\lambda B^k]$ and $A > \eta G^2/[\lambda B^k]$.
\subsection{$A \leq \eta G^2/[\lambda B^k]$}
On event $E_1$, we can bound
\[
Z_t^k \leq  \|\f_t^k-\fb_t^k\| \| \z_t^k - \w_*  \| \leq  \frac{\eta}{4} \|\f_t^k-\fb_t^k\|^2 +  \frac{1}{\eta} \|\z_t^k - \w_* \|^2 \leq   \frac{\eta}{4}  C^2 + \frac{1}{\eta} \|\z_t^k - \w_* \|^2. \\
\]
Summing up over all $t=1,2,\ldots,M$,
\begin{equation} \label{eqn:bound:1}
\sum_{t=1}^M Z_t^k \leq  \frac{\eta M C^2}{4}  + \frac{1}{\eta} \sum_{t=1}^M \|\z_t^k - \w_* \|^2  \leq \frac{\eta M C^2}{4} + \frac{G^2}{\lambda  B^k }.
\end{equation}
\subsection{$A  > \eta G^2/[\lambda B^k]$}
Similar to the above proof, on event $E_1$, we bound
\[
|Z_t^k| \leq  \|\f_t^k-\fb_t^k\| \| \z_t^k - \w_*  \| \leq  \frac{1}{\theta} \|\f_t^k-\fb_t^k\|^2 + \frac{\theta}{4} \|\z_t^k - \w_* \|^2 \leq \frac{C^2}{\theta}    + \frac{\theta A}{4},
\]
where $\theta$ can be any nonnegative real number. Denote the sum of conditional variances by
\[
\begin{split}
\Sigma_M^2=\sum_{t=1}^M \E^{t-1}_k \left[ [Z_t^k]^2\right] \leq C^2 \sum_{t=1}^M \|\z_t-\w_*\|^2 =C^2 A,
\end{split}
\]
where $\E_{k}^{t-1} [\cdot]$ denote the expectation conditioned on all the randomness up to the $t-1$-th iteration in the $k$-th epoch.

Notice that $A$ in the upper bound for $|Z_t^k|$ and $\Sigma_M^2$ is a random variable, thus we cannot directly apply Theorem~\ref{thm:bernstein}. To address this challenge, we make use of the peeling technique described in~\citep{Local_RC}, and have
\[
\begin{split}
& \Pr\left(\sum_{t=1}^M  Z_t^k\geq 2 \sqrt{C^2 A \tau} + \frac{4}{3} \left( \frac{C^2}{\theta} + \frac{\theta A}{4} \right) \tau \right) \\
= & \Pr\left(\sum_{t=1}^M  Z_t^k \geq 2\sqrt{C^2 A \tau} + \frac{4}{3} \left( \frac{C^2}{\theta} + \frac{\theta A}{4} \right) \tau, \frac{\eta G^2}{\lambda B^k} < A \leq  \frac{4MG^2}{\lambda^2}\right) \\
 =  & \Pr\left(\sum_{t=1}^M  Z_t^k\geq 2 \sqrt{C^2 A \tau} + \frac{4}{3} \left( \frac{C^2}{\theta} + \frac{\theta A}{4} \right) \tau, \right. \\
  & \qquad \left. \max_t |Z_t^k| \leq \frac{C^2}{\theta}  + \frac{\theta A}{4} ,\Sigma_M^2 \leq C^2 A,  \frac{\eta G^2}{\lambda B^k} < A \leq  \frac{4MG^2}{\lambda^2} \right) \\
 \leq & \sum_{i=1}^n \Pr\left(\sum_{t=1}^M  Z_t^k \geq 2\sqrt{C^2 A \tau} +\frac{4}{3} \left( \frac{C^2}{\theta} + \frac{\theta A}{4} \right) \tau,  \right. \\
  & \qquad \quad \ \ \left.\max_t |Z_t^k| \leq \frac{C^2}{\theta}  + \frac{\theta A}{4},  \Sigma_M^2 \leq C^2  A, \frac{\eta G^2 }{\lambda B^k} 2^{i-1} < A  \leq \frac{\eta G^2 }{\lambda B^k}2^{i} \right) \\
\leq  & \sum_{i=1}^n \Pr\left(\sum_{t=1}^M  Z_t^k \geq 2\sqrt{ \left(C^2 \frac{\eta G^2 }{\lambda B^k} 2^{i-1}\right)\tau} + \frac{4}{3} \left( \frac{C^2}{\theta} + \frac{\theta }{4} \frac{\eta G^2 }{\lambda B^k} 2^{i-1}\right)  \tau, \right. \\
  &\qquad \quad \ \ \left. \max_t |Z_t^k| \leq \frac{C^2}{\theta}  + \frac{\theta }{4} \frac{ \eta G^2 }{\lambda B^k}2^i , \Sigma_M^2 \leq C^2 \frac{\eta G^2 }{\lambda B^k} 2^i\right) \\
\leq  & \sum_{i=1}^n \Pr\left(\sum_{t=1}^M  Z_t^k \geq \sqrt{2 \left(C^2 \frac{\eta G^2}{\lambda B^k}  2^{i}\right)\tau} + \frac{2}{3} \left( \frac{C^2}{\theta} + \frac{\theta }{4} \frac{\eta G^2 }{\lambda B^k} 2^{i} \right)  \tau, \right. \\
  &\qquad \quad \ \ \left. \max_t |Z_t^k| \leq \frac{C^2}{\theta}  + \frac{\theta }{4} \frac{\eta G^2}{\lambda B^k}  2^i, \Sigma_M^2 \leq C^2 \frac{\eta G^2 }{\lambda B^k} 2^i \right) \\
 \leq  & ne^{-\tau},
\end{split}
\]
where
\[
n = \left \lceil \log_2 \frac{4MB^k }{\eta \lambda} \right\rceil,
\]
and the last step follows the Bernstein inequality for martingales in Theorem~\ref{thm:bernstein}. Setting
\[
\theta = \frac{3 \lambda}{4 \tau}, \textrm{ and } \tau = \log \frac{4n}{\tilde{\delta}},
\]
with a probability at least $1-\tilde{\delta}/4$ we have
\begin{equation} \label{eqn:bound:2}
\begin{split}
& \sum_{t=1}^M  Z_t^k \\
\leq &  2 \sqrt{C^2 A \tau} + \frac{4}{3} \left(\frac{C^2}{\theta} +\frac{\theta A}{4} \right) \tau  =  2 \sqrt{C^2 A \tau} + \frac{16C^2}{9 \lambda} \tau^2 +\frac{\lambda A}{4} \\
\leq &   \frac{4}{\lambda} C^2  \tau+ \frac{\lambda A}{4} + \frac{16C^2}{9 \lambda}\tau^2 +\frac{\lambda A}{4} =  \frac{4 C^2}{\lambda} \left( \log \frac{4n}{\tilde{\delta}} + \frac{4}{9} \log^2 \frac{4n}{\tilde{\delta}} \right) +\frac{\lambda A}{2}.
\end{split}
\end{equation}
We complete the proof by combining (\ref{eqn:bound:1}) and (\ref{eqn:bound:2}).
\end{document}